\documentclass{article}

\usepackage[preprint]{neurips_2021}

\usepackage[utf8]{inputenc} 
\usepackage[T1]{fontenc}    
\usepackage{comment}
\usepackage{url}            
\usepackage{booktabs}       
\usepackage{amsfonts}       
\usepackage{nicefrac}       
\usepackage{microtype}      
\usepackage[svgnames]{xcolor}
\usepackage{smile}
\usepackage{enumerate,natbib,mathrsfs}
\usepackage{algorithm}
\usepackage{algorithmic}
\usepackage{tablefootnote}
\usepackage{extarrows}
\usepackage[OT1]{fontenc}
\usepackage[bookmarks=false]{hyperref}

\hypersetup{
  pdftex,
  pdffitwindow=true,
  pdfstartview={FitH},
  pdfnewwindow=true,
  colorlinks,
  linktocpage=true,
  linkcolor=Red,
  urlcolor=Red,
  citecolor=Blue
}
\usepackage{stmaryrd}
\usepackage{graphicx}

\newcommand{\KL}{\mathrm{KL}}

\newcommand{\E}{\mathbb E}

\DeclareMathOperator{\poly}{poly}
\newcommand{\BR}{\mathfrak{BR}}
\newcommand{\mix}{\text{mix}}
\newcommand{\R}{\mathfrak{R}}

\title{Information Directed Sampling for Sparse Linear Bandits}

\author{%
  Botao Hao\\
  Deepmind\\
  \texttt{haobotao000@gmail.com} \\
   \And
   Tor Lattimore \\
   Deepmind \\
   \texttt{lattimore@google.com} \\
   \AND
  Wei Deng\\
  Department of Mathematics\\
   Purdue University\\
   \texttt{deng106@purdue.edu} \\
}

\begin{document}

\maketitle

\begin{abstract}
Stochastic sparse linear bandits offer a practical model for high-dimensional online decision-making problems and have a rich information-regret structure. In this work we explore the use of information-directed sampling (IDS), which naturally balances the information-regret trade-off. We develop a class of information-theoretic Bayesian regret bounds that nearly match existing lower bounds on a variety of problem instances, demonstrating the adaptivity of IDS. To efficiently implement sparse IDS, we propose an empirical Bayesian approach for sparse posterior sampling using a spike-and-slab Gaussian-Laplace prior.  Numerical results demonstrate
significant regret reductions by sparse IDS
relative to several baselines.

\end{abstract}

\section{Introduction}
Standard linear bandits associate each action with a feature vector and assume the mean reward is
the inner product between the feature vector and an unknown parameter vector
\citep{auer2002using, dani2008stochastic, rusmevichientong2010linearly, chu2011contextual, abbasi2011improved}. Sparse linear bandits generalize linear bandits by assuming the unknown parameter vector is sparse \citep{abbasi2012online, carpentier2012bandit, hao2020high} and is of great practical significance for modeling high-dimensional online decision-making problems \citep{bastani2020online}.

\citet[\S 24.3]{lattimore2018bandit} established a $\Omega(\sqrt{sdn})$ regret lower bound for the \emph{data-rich regime}, where $n$ is the horizon, $d$ is the feature dimension, $s$ is the sparsity and data-rich regime refers to the horizon $n\geq d^{\alpha}$ for some $\alpha>0$. This means polynomial dependence on $d$ is generally not avoidable without additional assumptions. However, this bound hides much of the rich structure of sparse linear bandits by a crude maximisation over all environments. 

When the action set admits a well-conditioned exploration distribution, \cite{hao2020high} discovered the information-regret trade-off phenomenon by establishing an $\Theta(\poly(s)n^{2/3})$ minimax rate for the \emph{data-poor regime}. An interpretation for this optimal rate is that the agent needs to acquire enough information for fast sparse learning by pulling informative actions that even have high regret. Explore-then-commit algorithm can achieve this rate for the data-poor regime but is sub-optimal in the data-rich regime. 
Therefore, our goal is to develop an efficient algorithm that can adapt to different information-regret structures for sparse linear bandits.

\paragraph{Contributions} Our contribution is three-fold:
\begin{itemize}
    \item We prove that optimism-based algorithms fail to optimally address the information-regret trade-off in sparse linear bandits, which results in a sub-optimal regret bound.
    \item We provide the first analysis using information theory for sparse linear bandits and derive a class of nearly optimal Bayesian regret bounds for IDS that can adapt to information-regret structures. 
    \item To approximate the information ratio, we develop an empirical Bayesian approach for sparse posterior sampling using spike-and-slab Gaussian-Laplace prior. Through several experiments, we justify the great empirical performance of sparse IDS with an efficient implementation.
\end{itemize}

\section{Preliminary}

We first introduce the basic setup of stochastic sparse linear bandits. The agent receives a compact action set $\cA\subseteq \mathbb R^d$ in the beginning where $|\cA|=K$. At each round $t$, the agent chooses an action $A_t\in\cA$ and receives a reward $Y_t = \langle A_t, \theta^*\rangle + \eta_t,$
where $(\eta_t)_{t=1}^n$ is a sequence of independent standard
Gaussian random variables and $\theta^*\in\mathbb R^d$ is the true parameter unknown to the agent. We make the mild boundedness assumption that for all $a \in \cA$, $\|a\|_\infty \leq 1$. The notion of sparsity can be defined through the parameter space $\Theta$:
\begin{equation*}
    \Theta = \left\{\theta\in\mathbb R^d \Bigg| \sum_{j=1}^d\ind\{\theta_j\neq 0\}\leq s, \|\theta\|_2\leq 1\right\}\,.
\end{equation*}
We assume $s$ is known and it can be relaxed by putting a prior on it. We consider the Bayesian setting where $\theta^*$ is a random variable taking values in $\Theta$ and denote $\rho$ as the prior distribution.  The optimal action is $x^* = \argmax_{a\in\cA}\mathbb E[\langle a, \theta^*\rangle|\theta^*]$. 
The agent chooses $A_t$ based on the history $\cF_t= (A_1, Y_1, \ldots, A_{t-1}, Y_{t-1})$. Let $\cD(\cA)$ be the space of probability measures over $\cA$. A policy $\pi = (\pi_t)_{t\in\mathbb N}$ is a sequence of deterministic functions where $\pi_t(\cF_t)$ specifies a probability distribution over $\cA$. 
The information-theoretic Bayesian regret of a policy $\pi$ \citep{russo2014learning}  is defined as
\begin{equation*}
   \BR(n; \pi)= \mathbb E\left[\sum_{t=1}^n \langle x^*, \theta^*\rangle - \sum_{t=1}^n Y_t\right] \,, 
\end{equation*}
where the expectation is over the interaction sequence induced by the agent and environment and the prior distribution over $\theta^*$.

\paragraph{Notation}  Denote $I_d$ as the $d\times d$ identity matrix. Let $[n] = \{1,2, \ldots, n\}$. For a  vector $x$ and positive semidefinite matrix $A$, we let $\|x\|_A=\sqrt{x^{\top}Ax}$ be 
the weighted $\ell_2$-norm and $\sigma_{\min}(A)$ be the minimum eigenvalue of $A$. The relation $ x \gtrsim y$ means that $x$ is greater or equal to $y$ up to some universial constant and $\tilde{O}(\cdot)$ hides modest logarithmic factors and universal constant.
The cardinality of a set $\cA$ is denoted by $|\cA|$. Given a measure $\mathbb P$ and jointly distributed random variables $X$ and $Y$ we let $\mathbb P_X$ denote the law of $X$ and we let $\mathbb P_{X|Y}$ be the conditional law of $X$ given $Y$: $\mathbb P_{X|Y}(\cdot) = \mathbb P(X\in \cdot|Y)$. The mutual information between $X$ and $Y$ is $I(X;Y) = \mathbb E[D_{\KL}(\mathbb P_{X|Y}||\mathbb P_X)]$ where $D_{\KL}$ is the relative entropy. 
We write $\mathbb P_t(\cdot) = \mathbb P(\cdot|\cF_t)$ as the posterior measure where $\mathbb P$ is the probability measure over $\theta$ and the history and $\mathbb E_t(\cdot) = \mathbb E(\cdot|\cF_t)$. Denote $I_t(X;Y) = \mathbb E_t[D_{\KL}(\mathbb P_{t, X|Y}||\mathbb P_{t,X})]$.

\section{Related work}

\begin{table}\label{table:comparsion}
\centering
\caption{Comparisons with existing results. APS11, LCS15, HLW20 refer to \cite{abbasi2012online, lattimore2015linear, hao2020high} accordingly. Exploratory action set is defined in Definition \ref{def:explortory} and $K$ is the number of actions. The last lower bound is developed in \cite{hao2020high}.}

\scalebox{0.94}{
\begin{tabular}{ |l|c|c|c|c| } 
 \hline & Action set& Algorithm & Type & Rate\\ 
 \hline
 APS11 & arbitrary & online-to-confidence &  freq & $O(\sqrt{sdn})$\\
  \hline
 LCS15
  &hypercube & elimination & freq &$O(s\sqrt{n})$\\
 \hline
 HLW20
  &exploratory & explore-then-commit & freq &$O(s^{2/3}n^{2/3})$\\
 \hline
  This paper
  &arbitrary & sparse IDS & Bayesian &$O(\min(\sqrt{sdn},\sqrt{dn\log(K)}))$\\
 \hline
   This paper
  &arbitrary & sparse TS & Bayesian &$O(\min(\sqrt{sdn},\sqrt{dn\log(K)}))$\\
 \hline
  This paper
  &exploratory & sparse IDS & Bayesian &$O(\min(sn^{2/3},\sqrt{sdn})$\\
 \hline
   Lower bound
  &arbitrary & NA & minimax &$\Omega(\sqrt{sdn})$\\
 \hline
    Lower bound
  &exploratory & NA & minimax &$\Omega(\min(s^{1/3}n^{2/3},\sqrt{dn}))$\\
 \hline
\end{tabular}}
\end{table}
\paragraph{Sparse linear bandits} \cite{abbasi2012online} proposed an inefficient
online-to-confidence-set conversion approach that achieves an $\tilde{O}(\sqrt{sdn})$ upper bound for an arbitrary action set. \cite{lattimore2015linear}
developed a selective explore-then-commit algorithm that only works when the action set is exactly the binary hypercube and derived an optimal $O(s\sqrt{n})$ upper bound. \cite{hao2020high} introduced the notion of an exploratory action set and proved a $\Theta(\poly(s)n^{2/3})$ minimax rate for the data-poor regime using an explore-then-commit algorithm. \cite{hao2021online} extended this concept to a MDP setting. \cite{carpentier2012bandit} considered a special case where the action set is the unit sphere and the noise 
is vector-valued so that the noise becomes smaller as the dimension grows. We summarize the comparison of existing results with our work in Table \ref{table:comparsion}.

\paragraph{Sparse linear contextual bandits} It recently
became popular to study the contextual setting, where the action set changes from round to round. These results can not be reduced to our setting since they rely on either careful assumptions on the context distribution \citep{bastani2020online, wang2018minimax, kim2019doubly, wang2020nearly, ren2020dynamic, oh2020sparsity} such that classical high-dimensional statistics can be used, or have polynomial dependency on the number of actions \citep{agarwal2014taming, foster2020beyond, simchi2020bypassing}.


\paragraph{Information-directed sampling} \cite{russo2018learning} introduced IDS and derived Bayesian regret bounds for multi-armed bandits, linear bandits and combinatorial bandits. \cite{liu2018information} studied IDS for bandits with graph-feedback. \cite{kirschner2018information, kirschner2020information} investigated the use of IDS for bandits with heteroscedastic noise and partial monitoring. \cite{kirschner2020asymptotically} proved the asymptotic optimality of frequentist IDS for linear bandits.

\paragraph{Information-theoretic analysis} \cite{russo2014learning} introduced an information-theoretic analysis of Thompson sampling (TS) and \cite{bubeck2020first} strengthened the result with a first-order Bayesian regret analysis. \cite{dong2018information, dong2019performance} extended the analysis to infinite-many actions and logistic bandits. \cite{lattimore2019information} explored the use of information-theoretic analysis for partial monitoring. \cite{lu2019information} generalized the analysis to reinforcement learning.

\paragraph{Bayesian sparse linear regression.}  In the Bayesian framework, spike-and-slab methods are commonly used as probabilistic tools for sparse linear regression but most of prior works focus on variable selection and parameter estimation rather than uncertainty quantification \citep{Mitchell1988, george1993variable, rovckova2018spike}. \cite{bai2020spike} provided a comprehensive overview.

\section{Does the optimism optimally balance information and regret?}\label{sec:tradeoff}
We demonstrate the necessity of balancing the trade-off between information and regret through a simple sparse linear bandit instance. We give an example where the optimal regret is only possible by playing actions that are known to be sub-optimal. This phenomenon has disturbing
implications for policies based on the principle
of optimism, which
is that they can never be minimax
optimal in certain regime.

\paragraph{Illustrative example} Consider a problem instance where $\cA=\cI\cup\cU$ is the union of an informative action set $\cI$ and an uninformative action set $\cU$:
\begin{itemize}
    \item $\cI$ is a \emph{subset} of  the hypercube that has three properties. First, $|\cI|=O(s\log(ed/s))$. Second, the last coordinate of actions in $\cI$ is always -1. Third, the empirical covariance of the uniform distribution over $\cI$ has a restricted minimum eigenvalue (Definition \ref{def:re}) at least 1/4. We prove such $\cI$ does exist in Appendix \ref{sec:proof_claim} through a probabilistic argument.
   \item $\cU=\{x\in\mathbb R^d|x_j\in\{-1, 0, 1\} \ \text{for} \ j\in[d-1], \|x\|_1=s-1, x_d = 0\}.$
\end{itemize}
The true parameter $\theta^* = (\varepsilon, \ldots, \varepsilon, 0, \ldots, 0, -1),$
where $\varepsilon>0$ is a small constant. 

\paragraph{Information-regret structure} Sampling an action uniformly at random from $\cI$ ensures the covariance matrix is well-conditioned so that sparse learning such as Lasso \citep{tibshirani1996} can be used for learning $\theta^*$ faster than ordinary least squares. This means pulling actions from $\cI$ provides more information to infer $\theta^*$ than from $\cU$. On the other hand, actions from $\cI$ lead to high regret due to the last coordinate -1. As a consequence, \cite{hao2020high} has proven that the minimax regret for this problem is $\Theta(\poly(s)n^{2/3})$ when the horizon is smaller than the ambient dimension. 

\paragraph{Sub-optimality of optimism-based algorithms}
We argue the optimism principle does not take this subtle trade-off into consideration and yields sub-optimal regret in the data-poor regime. In general, optimism-based algorithms choose $A_t = \argmax_{a\in\cA}\max_{\tilde{\theta}\in\cC_t}\langle a, \tilde{\theta}\rangle,
$
where $\cC_t\subseteq \mathbb R^d$ is a confidence set that contains the true $\theta^*$ with high probability. We assume that there exists a constant $c>0$ such that
$\cC_t\subseteq \{\theta:(\hat{\theta}_t-\theta)^{\top}V_t(\hat{\theta}_t-\theta)\leq c\sqrt{s\log(n)}\},
$
where $V_t=\sum_{s=1}^tA_sA_s^{\top}$ and $\hat{\theta}_t$ is some sparsity-aware estimator. Such confidence set can be constructed through an online-to-confidence set conversion approach \citep{abbasi2012online}. Define $ \R_{\theta}(n; \pi)= \mathbb E[\sum_{t=1}^n \langle x^*, \theta\rangle - \sum_{t=1}^n Y_t]$ for a fixed $\theta$.
\begin{claim}\label{claim:sub-opt}
Let $\pi^{\text{opt}}$ be such an optimism-based algorithm. There exists a sparse linear bandit instance characterized by $\theta$ such that for the data-poor regime, we have $$
\R_{\theta}(n;\pi^{\text{opt}}) \gtrsim n/(\log(n)s\log(ed/s))\,.$$
\end{claim}
The proof is deferred to Appendix \ref{sec:proof_claim}. The reason is that optimism-based algorithms do not choose actions for which they have collected enough statistics to prove these actions are suboptimal, but in the sparse linear setting it can be worth playing these actions when they are informative about other actions for which the statistics are not yet so clear. This phenomenon has been observed before in linear and structured bandits \citep{lattimore2017end, combes2017minimal, hao2019adaptive}.

\section{Information-directed sampling}

As shown by \cite{hao2020high}, although the explore-then-commit algorithm can achieve the minimax optimal regret in the data-poor regime, it suffers sub-optimal regret in the data-rich regime. This motivates us consider IDS.

\subsection{Design principle of IDS}\label{subsec:design_principle}
Unlike the optimism principle, IDS explicitly balances the amount of information it gains about the optimal action and expected single-round regret through minimizing a notion of information ratio. More formally, when playing action $a$, the \emph{information gain} $I_t(x^*; Y_{t,a})$ is the mutual information between the optimal action and the reward the agent receives for taking action $a$, and the expected  single-round regret is $\Delta_t(a):=\mathbb E_t[\langle x^*, \theta^*\rangle-\langle a, \theta^*\rangle]$. 
The information ratio is defined as
\begin{equation}\label{eqn:IDS}
     \Psi_{t,\lambda}(\pi) = \frac{(\Delta_t^{\top}\pi)^{\lambda}}{I_t^{\top}\pi}\,,
\end{equation}
where we write $\Delta_t\in\mathbb R^{|\cA|}$ and $I_t\in\mathbb R^{|\cA|}$ as corresponding vectors. Then IDS takes the action according to $ \pi_{t} = \argmin_{\pi\in\cD(\cA)}\Psi_{t,2}(\pi)$.
\begin{remark}
The information ratio defined in Eq.~\eqref{eqn:IDS} is a little more general than what \cite{russo2018learning} introduced which specified $\lambda=2$. As observed by \cite{lattimore2020mirror}, the right value of $\lambda$ depends on the dependence of the regret on the horizon. 
\end{remark}

\subsection{Information-theoretic Bayesian regret bound}
In this section, we derive a class of Bayesian regret upper bound for sparse IDS. We first define a notion of exploratory action set. 
\begin{definition}[Exploratory action set]\label{def:explortory}
Let $C_{\min}(\cA) = \max_{\mu\in\cD(\cA)} \sigma_{\min}(\mathbb E_{A\sim \mu}[AA^{\top}]).$ For an action set $\cA$, if $C_{\min}(\cA)\geq 1$, we say $\cA$ is exploratory.
\end{definition}

We say that $\cA$ has sparse optimal actions if the optimal action is $s$-sparse almost surely with respect to the prior. One can verify the action set of the hard instance developed in \cite{hao2020high}\footnote{The hard instance is almost the same as the one in illustrative example except the informative action set is a full hypercube.} is exploratory and has sparse optimal actions since sampling uniformly from the corner of informative action set shows that $C_{\min}(\cA) \geq 1$ and the optimal actions always come from uninformative action set, which is sparse.
\begin{theorem}[Regret bound for sparse IDS]\label{thm:IDS}
Suppose $\pi^{\text{IDS}} = (\pi_{t})_{t\in \mathbb N}$ where $\pi_{t} = \argmin_{\pi}\Psi_{t,2}(\pi)$. Let $\Delta = \min(\log(K), 2s\log(Cdn^{1/2}/s))$ for some absolute constant $C>0$. For an arbitrary action set, the following regret bound holds
\begin{equation*}
     \BR(n;\pi^{\text{IDS}}) \leq \sqrt{\frac{1}{2}nd\Delta}\,.
\end{equation*}
When $\cA$ is exploratory and has sparse optimal actions, the following regret bound holds
\begin{equation*}
\begin{split}
     \BR(n;\pi^{\text{IDS}}) \leq \min\left\{ \sqrt{\frac{1}{2}nd\Delta},\frac{s^{\frac{2}{3}}n^{\frac{2}{3}}\Delta^{\frac{1}{3}}}{(2C_{\min}(\cA))^{\frac{1}{3}}}\right\}\,.
\end{split}
\end{equation*}
\end{theorem}
This theorem shows the great adaptivity of IDS for sparse linear bandits in the sense that a single policy adapts to different information-regret structures. We summarize the regret bounds in a variety of of different regimes in Table \ref{table:IDS_result}.

\begin{remark}
The explore-then-comment algorithm proposed by \cite{hao2020high} for sparse linear bandits has $O(\poly(s)n^{2/3})$ regret bound when the action set is exploratory and it is known that this $O(n^{2/3})$ rate is not improvable. Thus, it is sub-optimal for data-rich regime comparing with $\Theta(\sqrt{sdn})$ minimax rate. In contrast, IDS is nearly optimal in both regimes. 
\end{remark}

\begin{table}\label{table:IDS_result}
\centering
\caption{Summary of regret bounds of IDS for different regimes. Data-rich regime refers to $n \gtrsim d^3\Delta/s^4$ and large $K$ refers to $K \gtrsim d\exp(s)$.}

\scalebox{0.94}{
\begin{tabular}{ |l|c|c|c| } 
 \hline & Arbitrary action set& Exploratory (data-rich) & Exploratory (data-poor)\\ 
 \hline
Large $K$ & $O(\sqrt{nds})$ & $O(\sqrt{nds})$ &  $O(sn^{2/3})$ \\
  \hline
Small $K$
  &$O(\sqrt{nd\log(K)})$ & $O(\sqrt{nd\log(K)})$ & $O(s^{2/3}n^{2/3}\log^{1/3}(K))$\\
 \hline
\end{tabular}}

\end{table}
As a direct application of our analysis, we also include a novel Bayesian regret bound for sparse TS.
\begin{corollary}[Regret bound for sparse TS]
For an arbitrary action set, the following regret bound holds for some absolute constant $C>0$
\begin{equation*}
     \BR(n;\pi^{\text{TS}}) \leq \sqrt{\frac{1}{2}nd\min(\log(K), 2s\log(Cdn^{1/2}/s))}\,.
\end{equation*}
\end{corollary}

\begin{proof}[Proof of Theorem \ref{thm:IDS}]
We prove our main result in three steps. All the proofs of technical lemmas are deferred to the appendix.

\paragraph{Step 1: Generic Bayesian regret upper bound} We define $\Psi_{*,\lambda}\in \mathbb R$ as the \emph{worse-case information ratio} such that for each $t\in[n]$, $\Psi_{t,\lambda}(\pi_t)\leq \Psi_{*,\lambda}$ almost surely. 
\begin{lemma}\label{lemma:generic}
Suppose $\pi^{\text{IDS}} = (\pi_{t})_{t\in \mathbb N}$ where $\pi_{t} = \argmin_{\pi}\Psi_{t,2}(\pi)$.  Then the following regret bound holds
\begin{equation*}
    \BR(n;\pi^{\text{IDS}}) \leq \inf_{\lambda\geq 2}2^{1-2/\lambda}\Psi^{1/\lambda}_{*,\lambda}I(x^*; \cF_{n+1})^{1/\lambda}n^{1-1/\lambda}\,,
\end{equation*}
where $\cF_{n+1}$ refers to the history.
\end{lemma}
This lemma demonstrates the adaptivity of a single IDS for different information ratios. The choice of $\lambda$ essentially trades off the information-ratio and the horizon. 

\paragraph{Step 2: Bounding the worse-case information ratio} We bound the worse-case information ratio for different $\lambda$. It  shows that for certain action sets, the worse-case information ratio with $\lambda=3$ could be much smaller than the one with $\lambda=2$.
\begin{lemma}\label{lemma:information_ratio}
For an arbitrary action set, we have $ \Psi_{*,2}\leq d/2$. For an exploratory action set that has sparse optimal actions, we have $\Psi_{*,3}\leq s^2/(4C_{\min}(\cA)).
$
\end{lemma}
The bound of $\Psi_{*,2}$ essentially follows  \citet[Proposition 5]{russo2014learning} that bounds the information ratio of IDS by TS. And $\Psi_{*,3}$ is bounded by the information ratio of a mixture policy $\pi_t^{\mix} = (1-\gamma)\pi_t^{\text{TS}}+\gamma \mu$ where $\mu$ is an exploratory policy such that $\sigma_{\min}(\mathbb E_{A\sim \mu}[AA^{\top}])$ is a constant and the mixture rate $\gamma\geq 0$ is optimized to minimize the bound.
 
\paragraph{Step 3: Bounding the mutual information} The mutual information $I(x^*;\cF_{n+1})$ quantifies the cumulative information gain about the optimal action. \cite{russo2014learning,russo2018learning} naively bound this term by entropy $H(x^*)$, which can be arbitrarily large or even infinite for some priors. Instead, we bound this term by the mutual information between the true parameter and the history through data-processing lemma.
\begin{lemma}\label{lemma:bound_mutual_information}
$I(x^*; \cF_{n+1}) \leq I(\theta^*; \cF_{n+1})\leq \min\{\log(K), 2s\log(Cdn^{1/2}/s)\}$.
\end{lemma}
Our proof is based on the metric entropy of the parameter space and square root KL-divergence that is commonly used in information-theoretic lower bound analysis \citep{yang1999information}. As a by-product of our analysis, by setting $s=d$, our analysis recovers the $\tilde{O}(d\sqrt{n})$ Bayesian regret bound for TS under linear bandits with infinitely many actions without using rate-distortion theory \citep{dong2018information}. 
Combining Lemmas \ref{lemma:generic}-\ref{lemma:bound_mutual_information} yields our conclusion.
\end{proof}

\section{Computational methods}\label{sec:computation}
In this section, we provide an efficient implementation of sparse IDS. The main challenge is to generate posterior samples in a computationally efficient manner to approximate the information ratio. Due to the lack of conjugate prior, we propose an empirical Bayesian approach for sparse sampling with spike-and-slab priors.

\subsection{An empirical Bayesian approach for sparse sampling}

In the Bayesian framework, the golden standard for modeling $\theta^*$ is to place spike-and-slab priors \citep{Mitchell1988}. With a hierarchical structure over the parameter and model space, vanilla spike-and-slab priors usually have the following form
\begin{equation}\label{eqn:generic_form}
\begin{split}
     \rho(\theta|\bgamma, \sigma^2) = \prod_{j=1}^d\left[\gamma_j\psi_1(\theta_j, \sigma)+(1-\gamma_j)\psi_0(\theta_j,\sigma)\right],
     \rho(\bgamma|\beta) = \prod_{j=1}^d\beta^{\gamma_j}(1-\beta)^{1-\gamma_j}\,,
\end{split}
\end{equation}
where $\bgamma = (\gamma_1,\ldots, \gamma_d)^{\top}$ is an intermediate binary vector that indexs the $2^d$ possible models and $\beta\in[0, 1]$ denotes a priori fraction of relevant variables among all the parameters. In particular, $\psi_1(\theta, \sigma)$ serves as a slab distribution to models relevant variables and $\psi_0(\theta, \sigma)$ is a point mass at zero that serves as a spike distribution to model irrelevant variables.
\begin{remark}
It is typical to assume $\beta$ follows a Beta prior as  $\beta\sim\text{Beta}(a, b)$ and variance $\sigma^2$ follows an inverse gamma prior $\rho(\sigma^2)= \text{IG}(\nu/2, \nu\lambda/2)$ with $\nu=1$ and $\lambda=1$ (\cite{rovckova2014spike}). For simplicity, we do not impose those additional layers of priors.
\end{remark}

\textbf{Prior specification}. Although the prior in Eq.~\eqref{eqn:generic_form} is theoretically sound, exploring the full posterior in high-dimensions over the entire model space using point-mass spike-and-slab priors can be computationally prohibitive. Therefore, the spike distribution is usually relaxed as a small-scale Gaussian distribution \citep{rovckova2014spike} or Laplace distribution \citep{rovckova2018spike}. Thus, we consider a spike-and-slab Gaussian-Laplace prior that inherits the property of the Lasso \citep{tibshirani1996} for variable selection while the Gaussian component avoids the potential over-shrinkage effect. Each component of the prior is specified as
\begin{equation*} 
    \psi_0(\theta, \sigma)=\frac{1}{2\sigma\lambda_0}\exp\left(-\frac{|\theta|}{\sigma\lambda_0}\right), \psi_1(\theta, \sigma)=\frac{1}{\sqrt{2\pi\sigma^2\lambda_1}}\exp\left(-\frac{\theta^2}{2\sigma^2\lambda_1}\right)\,,
\end{equation*}
 where $\lambda_0>0$ denotes a scaling parameter that encourages the shrinkage of irrelevant parameters and $\lambda_1$ is often set to a large value for a standard regularization \citep{rovckova2014spike}.

\textbf{An empirical Bayesian approach.} Suppose $\cL(\cF_{n+1}|\theta,\sigma^2)$ is the likelihood function where $Z_n$ is the historical data. According to the Bayes rule, the full posterior follows
\begin{equation}\label{eqn:full_posterior}
    p(\theta, \bgamma|\cF_{n+1}, \sigma^2, \beta)\propto \cL(\cF_{n+1}|\theta, \sigma^2)\rho(\theta|\bgamma,\sigma^2)\rho(\bgamma|\beta)\,.
\end{equation}
To speed up the sampling, we only sample $\theta$ and optimize $\bgamma$ with respect to the posterior instead. To tackle this issue of the binary vector $\bgamma$, we consider a continuous relaxation by introducing a latent vector $\nu\in[0,1]^d$ as the probability of the variable being included in the model. We focus on the simulations of the conditional expectation of the complete posterior:
\begin{equation*}
\E_{\bgamma|\cdot}\left[\log p(\theta, \bgamma|\cF_{n+1}, \sigma^2, \beta)\right]=\log \cL(\cF_{n+1}|\theta,\sigma^2)+\E_{\bgamma|\cdot}\left[\log\rho(\bgamma|\beta)+\log\rho(\theta|\bgamma,\sigma^2)\right]+C_1\,,
\end{equation*}
where $C_1$ is the normalizing constant and $\E_{\bgamma|\cdot}[\cdot]$ denotes the the conditional expectation with respect to $\gamma$ given the current parameter $\theta$. Then we compute
\begin{equation*}
    \E_{\bgamma|\cdot}\left[\log\rho(\bgamma|\beta)\right]=\sum_{i=1}^d\E_{\bgamma|\cdot}\left[\gamma_i \log(\beta)+(1-\gamma_i)\log(1-\beta)\right]=\sum_{i=1}^d \log\left(\dfrac{\beta}{1-\beta}\right) \nu_i+C_2\,,
\end{equation*}
where $C_2=d\log(1-\beta)$ and $\nu_i=\E_{\bgamma|\cdot}[\gamma_{i}]$ denotes a conditional probability. By the Bayes rule, we have
\begin{equation}
\label{prob_non_sparse}
    \nu_i=\mathbb P(\gamma_i=1|\theta_i,\beta)=\frac{\rho(\theta_i|\gamma_i=1) \mathbb P(\gamma_i=1|\beta)}{\rho(\theta_i|\gamma_i=1) \mathbb P(\gamma_i=1|\beta)+\rho(\theta_i|\gamma_i=0) \mathbb P(\gamma_i=0|\beta)}\,.
\end{equation}
For the mixture prior, we optimize the variational lower bound
\begin{equation*}
\begin{split}
\mathbb E_{\bgamma|\cdot}[\log\rho(\theta|\bgamma,\sigma^2)]\geq
    \sum_{j=1}^d -\frac{1-\nu_j}{\lambda_0} \frac{|\theta_{j}|}{\sigma} - \frac{\nu_j}{\lambda_1}\frac{\theta_{j}^2}{{2\sigma^2}}+C_3\,,
\end{split}
\end{equation*}
where the inequality follows by Jensen's inequality and $C_3$ denotes a trivial constant.

We summarize the full sampling procedure in Algorithm \ref{EB_SSGL}. Given a current estimate of $(\theta^{(k)}, \nu^{(k)})$ at step $k$ and $\cF_{n+1}$, we adapt an empirical Bayesian method \citep{deng2019} by iteratively sampling $\theta$ based on the negative log-posterior with adaptive priors
\begin{equation*}
    Q(\theta|\theta^{(k)}, \nu^{(k)}, \cF_{n+1})=-\log \cL(\cF_{n+1}|\theta,\sigma^2)+\sum_{i=1}^d\left(\frac{1-\rho_i}{\lambda_0} \frac{|\theta_{i}|}{\sigma} + \frac{\rho_i}{\lambda_1}\frac{\theta_{i}^2}{{2\sigma^2}}\right)\,,
\end{equation*}
and optimizing the conditional probability $\nu$ through stochastic approximation algorithms (\cite{RobbinsM1951}) until the equilibrium is achieved. 

\begin{algorithm}[htb!]
    	\caption{Empirical Bayesian sparse sampling}
	\begin{algorithmic}[1]\label{EB_SSGL}
		\STATE
		\textbf{Input:} dataset $\cF_{n+1}$, learning rate $(\eta_k)$, step size $(\omega_k)$, a priori knowledge of $\sigma^2$ and $\beta$, thinning factor $T$, number of posterior samples $M$, regularization parameters $\lambda_0$, $\lambda_1$.
		\STATE \textbf{Initialize:} $\theta^{(0)}\sim N(0, 0.1 I_{d})$ and $\nu^{(0)} = (0.5, \ldots, 0.5)^{\top}$.
		\FOR{$k \geq 1$}
		\STATE
		Sampling step: $
		    \theta^{(k+1)}=\theta^{(k)} -\eta_k \frac{\partial }{\partial \theta} Q(\theta|\theta^{(k)}, \nu^{(k)}, \cF_{n+1}) + \sqrt{2\eta_k} \xi^{(k)},
		$
		 where $\xi^{(k)}$ is a standard Gaussian random vector.
		\STATE
		Stochastic approximation step:
		$
		\nu^{(k+1)}=(1-\omega_k) \nu^{(k)} + \omega_k \nu,
		$
 where $\nu$ is derived from Eq.~\eqref{prob_non_sparse}.
		\ENDFOR
			\STATE	\textbf{Output:} $M$ posterior samples $\theta^{(T)}, \theta^{(2T)}, \ldots, \theta^{(MT)}$.
	\end{algorithmic}
\end{algorithm}

\subsection{Optimize the information ratio}

For sparse linear bandits, it is expensive to estimate and optimize the original information ratio that involves the calculation of KL-divergence. Following Section 6.3 in \cite{russo2018learning}, we  optimize a variance-based information ratio instead:
$\pi_t = \argmin_{\pi}(\Delta_t^{\top}\pi)^2/(2v_t^{\top}\pi),
$
where we define $v_t(a) = \mathbb E_t[a^{\top}\mathbb E_t[\theta^*|x^*]-a^{\top}\mathbb E_t[\theta^*]]^2$ for each $a\in\cA$. With sufficient number of posterior samples of $\theta^*$ produced by Algorithm \ref{EB_SSGL}, we can accurately estimate $\mathbb E_t[\theta^*|x^*], \mathbb E_t[\theta^*]$ and the information ratio. The detailed procedure is deferred to Appendix \ref{sec:detailed_algorithm}. For the optimization step, we simply choose the action who can has the minimum per-action variance-based information ratio to accelerate the computation.  The overall algorithm can be found in Algorithm \ref{alg:ids_final}.
{\small
\begin{algorithm}[htb!]
	\caption{Sparse IDS}
	\begin{algorithmic}[1]\label{alg:ids_final}
		\STATE
		\textbf{Input:} time horizon $n$, action set $\cA$, number of posterior samples per round $M$.
		\FOR{$t= 1, \cdots, n$}
		\STATE Obtain $M$ posterior samples $\theta^1,\ldots, \theta^M$ by Algorithm \ref{EB_SSGL}.
		\STATE Calculate $\hat{\Delta}_t\in\mathbb R^{|\cA|}$ and $\hat{v}_t\in\mathbb R^{|\cA|}$ using Algorithm \ref{alg:est_inf_ratio} in Appendix \ref{sec:detailed_algorithm}.
\STATE Take the action $  A_t= \argmin_{a\in\cA}\hat{\Delta}_t^2(a)/\hat{v}_t(a)$ and receive a reward: $Y_t = \langle A_t, \theta^*\rangle + \eta_t.$
\ENDFOR
	\end{algorithmic}
\end{algorithm}
}

\section{Experiments}
 
\begin{figure}
 \centering
 \includegraphics[width=0.32\linewidth]{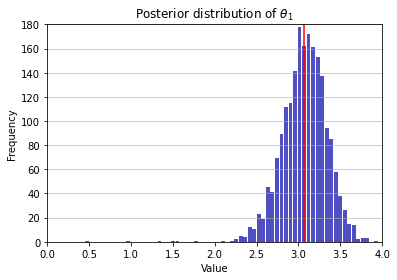}
  \includegraphics[width=0.32\linewidth]{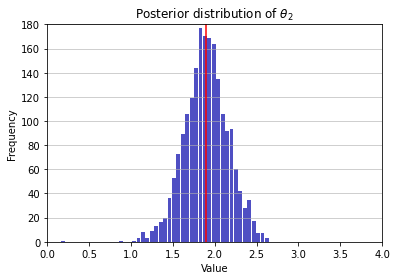}
     \includegraphics[width=0.32\linewidth]{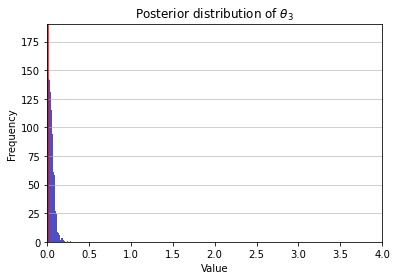}
\caption{Posterior distributions for the first three coordinates. The red lines are posterior means.}
\label{fig:posterior}
\end{figure}
First, we evaluate the performance of the empirical Bayesian sparse sampling procedure for generating posterior samples through an offline sparse linear regression. We set $d=10, s=3, n=100$ and the actions are drawn i.i.d from a multivariate normal distribution $N(0, \Sigma)$ with $\Sigma_{ij}= 0.6^{|i-j|}$. The true parameter is $\theta^* = (3,2,0,0,\ldots, 0)\in\mathbb R^{10}$. We plot the empirical posterior distributions as well as their posterior mean for the first three covariates in Figure \ref{fig:posterior}. It shows that the posterior distribution concentrates well around the true value and the algorithm identifies the sparse pattern quickly. 

Second, we evaluate sparse IDS with several other competitors. In particular, we compare with LinUCB \citep{abbasi2011improved}, LinTS with Gaussian prior \citep{agrawal2013thompson}, IDS for linear bandits (Algorithm 6 in \cite{russo2018learning}) and ESTC \citep{hao2020high}. Note that the first three algorithms are not sparsity-aware. Our sparse sampling procedure naturally induces a sparse TS algorithm so we include it into comparison. 

\paragraph{Setting} All the true parameters are randomly generated from a multivariate normal distribution, truncated to be sparse and normalized to have square norm 1. The noise variance is fixed to be 2 and we replicate the experiments over 200 trials. We plot the empirical cumulative Bayesian regret.  Each Bayesian algorithm will take 10000 posterior samples. We use the TS without blow-up factor for the variance and tune the length of confidence interval of LinUCB over a candidate set.

\paragraph{Hard sparse linear bandits instance} Consider the hard problem instance introduced in Section \ref{sec:tradeoff} that
includes informative and uninformative action sets and set $d=10, s=2$.  For each trial, we record the number of pulls of sparse TS and sparse IDS for informative actions. We draw the histogram of number of pulls during 200 trial in Figure \ref{fig:hard_instance_d5}. It is clear that IDS tends to invest more on the informative action sets but suffer less regret than TS if there exists an information-regret trade-off phenomenon. 
    \begin{figure}
 \centering
 \includegraphics[width=0.32\linewidth]{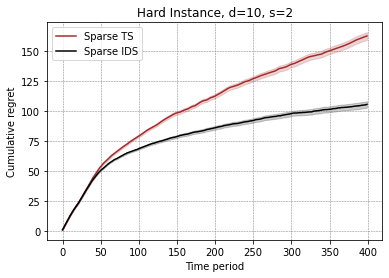}
  \includegraphics[width=0.32\linewidth]{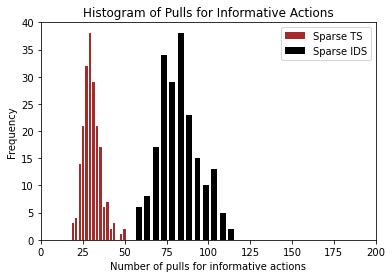}
\caption{The left panel is the cumulative regret and the right panel is the histogram of number of pulls for informative actions. It's clear that sparse TS does not value information enough as it should.}
\label{fig:hard_instance_d5}
\end{figure}

\paragraph{Multivariate Gaussian action set} We consider a more general case where each action is generated from multivariate normal distribution $N(0, \Sigma)$ with $\Sigma_{ij}= 0.6^{|i-j|}$. The number of actions $K$ is fixed to be 200 and the level of sparsity $s/d$ is fixed to be 0.1. We report the results in Figure \ref{fig:MG} for $d=20, 40, 100$. It is obvious that sparse IDS consistently outperforms other algorithms and the improvement increases as the feature dimension increases. ESTC performs better than non sparsity-aware algorithms in the data-poor regime but perform poorly for the data-rich horizon.

     \begin{figure}
 \centering
 \includegraphics[width=0.32\linewidth]{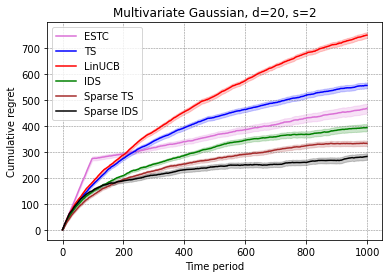}
  \includegraphics[width=0.32\linewidth]{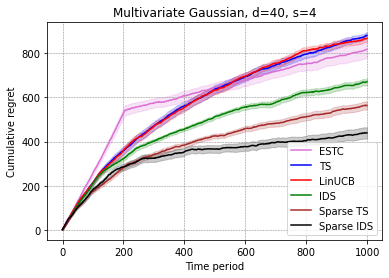}
    \includegraphics[width=0.32\linewidth]{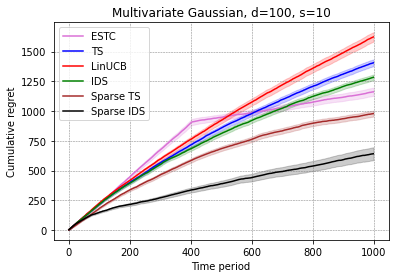}
\caption{Cumulative regret for $d=20, 40, 100$.}
\label{fig:MG}
\end{figure}

\section{Conclusion}
In this work, we investigate the theoretic and practical applicability of information-directed sampling for sparse linear bandits. An interesting future direction is to extend similar ideas to sparse linear contextual bandits.

\bibliographystyle{plainnat}
{\small
\bibliography{ref}

\begin{thebibliography}{48}
\providecommand{\natexlab}[1]{#1}
\providecommand{\url}[1]{\texttt{#1}}
\expandafter\ifx\csname urlstyle\endcsname\relax
  \providecommand{\doi}[1]{doi: #1}\else
  \providecommand{\doi}{doi: \begingroup \urlstyle{rm}\Url}\fi

\bibitem[Abbasi-Yadkori et~al.(2011)Abbasi-Yadkori, P{\'a}l, and
  Szepesv{\'a}ri]{abbasi2011improved}
Yasin Abbasi-Yadkori, D{\'a}vid P{\'a}l, and Csaba Szepesv{\'a}ri.
\newblock Improved algorithms for linear stochastic bandits.
\newblock In \emph{Advances in Neural Information Processing Systems}, pages
  2312--2320, 2011.

\bibitem[Abbasi-Yadkori et~al.(2012)Abbasi-Yadkori, Pal, and
  Szepesvari]{abbasi2012online}
Yasin Abbasi-Yadkori, David Pal, and Csaba Szepesvari.
\newblock Online-to-confidence-set conversions and application to sparse
  stochastic bandits.
\newblock In \emph{Artificial Intelligence and Statistics}, pages 1--9, 2012.

\bibitem[Agarwal et~al.(2014)Agarwal, Hsu, Kale, Langford, Li, and
  Schapire]{agarwal2014taming}
Alekh Agarwal, Daniel Hsu, Satyen Kale, John Langford, Lihong Li, and Robert
  Schapire.
\newblock Taming the monster: A fast and simple algorithm for contextual
  bandits.
\newblock In \emph{International Conference on Machine Learning}, pages
  1638--1646. PMLR, 2014.

\bibitem[Agrawal and Goyal(2013)]{agrawal2013thompson}
Shipra Agrawal and Navin Goyal.
\newblock Thompson sampling for contextual bandits with linear payoffs.
\newblock In \emph{International Conference on Machine Learning}, pages
  127--135. PMLR, 2013.

\bibitem[Auer(2002)]{auer2002using}
Peter Auer.
\newblock Using confidence bounds for exploitation-exploration trade-offs.
\newblock \emph{Journal of Machine Learning Research}, 3\penalty0
  (Nov):\penalty0 397--422, 2002.

\bibitem[Bai et~al.(2020)Bai, Rockova, and George]{bai2020spike}
Ray Bai, Veronika Rockova, and Edward~I George.
\newblock Spike-and-slab meets lasso: A review of the spike-and-slab lasso.
\newblock \emph{arXiv preprint arXiv:2010.06451}, 2020.

\bibitem[Bastani and Bayati(2020)]{bastani2020online}
Hamsa Bastani and Mohsen Bayati.
\newblock Online decision making with high-dimensional covariates.
\newblock \emph{Operations Research}, 68\penalty0 (1):\penalty0 276--294, 2020.

\bibitem[Bubeck and Sellke(2020)]{bubeck2020first}
S{\'e}bastien Bubeck and Mark Sellke.
\newblock First-order bayesian regret analysis of thompson sampling.
\newblock In \emph{Algorithmic Learning Theory}, pages 196--233. PMLR, 2020.

\bibitem[Carpentier and Munos(2012)]{carpentier2012bandit}
Alexandra Carpentier and R{\'e}mi Munos.
\newblock Bandit theory meets compressed sensing for high dimensional
  stochastic linear bandit.
\newblock In \emph{Artificial Intelligence and Statistics}, pages 190--198,
  2012.

\bibitem[Chu et~al.(2011)Chu, Li, Reyzin, and Schapire]{chu2011contextual}
Wei Chu, Lihong Li, Lev Reyzin, and Robert Schapire.
\newblock Contextual bandits with linear payoff functions.
\newblock In \emph{Proceedings of the Fourteenth International Conference on
  Artificial Intelligence and Statistics}, pages 208--214, 2011.

\bibitem[Combes et~al.(2017)Combes, Magureanu, and
  Proutiere]{combes2017minimal}
Richard Combes, Stefan Magureanu, and Alexandre Proutiere.
\newblock Minimal exploration in structured stochastic bandits.
\newblock In \emph{Advances in Neural Information Processing Systems}, pages
  1763--1771, 2017.

\bibitem[Dani et~al.(2008)Dani, Hayes, and Kakade]{dani2008stochastic}
Varsha Dani, Thomas~P Hayes, and Sham~M Kakade.
\newblock Stochastic linear optimization under bandit feedback.
\newblock 2008.

\bibitem[Deng et~al.(2019)Deng, Zhang, Liang, and Lin]{deng2019}
Wei Deng, Xiao Zhang, Faming Liang, and Guang Lin.
\newblock An adaptive empirical bayesian method for sparse deep learning.
\newblock \emph{Advances in neural information processing systems},
  2019:\penalty0 5563, 2019.

\bibitem[Dong and Roy(2018)]{dong2018information}
Shi Dong and Benjamin~Van Roy.
\newblock An information-theoretic analysis for thompson sampling with many
  actions.
\newblock In \emph{Proceedings of the 32nd International Conference on Neural
  Information Processing Systems}, pages 4161--4169, 2018.

\bibitem[Dong et~al.(2019)Dong, Ma, and Van~Roy]{dong2019performance}
Shi Dong, Tengyu Ma, and Benjamin Van~Roy.
\newblock On the performance of thompson sampling on logistic bandits.
\newblock In \emph{Conference on Learning Theory}, pages 1158--1160. PMLR,
  2019.

\bibitem[Foster and Rakhlin(2020)]{foster2020beyond}
Dylan Foster and Alexander Rakhlin.
\newblock Beyond ucb: Optimal and efficient contextual bandits with regression
  oracles.
\newblock In \emph{International Conference on Machine Learning}, pages
  3199--3210. PMLR, 2020.

\bibitem[George and McCulloch(1993)]{george1993variable}
Edward~I George and Robert~E McCulloch.
\newblock Variable selection via gibbs sampling.
\newblock \emph{Journal of the American Statistical Association}, 88\penalty0
  (423):\penalty0 881--889, 1993.

\bibitem[Hao et~al.(2020{\natexlab{a}})Hao, Lattimore, and
  Szepesvari]{hao2019adaptive}
Botao Hao, Tor Lattimore, and Csaba Szepesvari.
\newblock Adaptive exploration in linear contextual bandit.
\newblock \emph{AISTATS}, 2020{\natexlab{a}}.

\bibitem[Hao et~al.(2020{\natexlab{b}})Hao, Lattimore, and Wang]{hao2020high}
Botao Hao, Tor Lattimore, and Mengdi Wang.
\newblock High-dimensional sparse linear bandits.
\newblock \emph{Advances in Neural Information Processing Systems}, 33,
  2020{\natexlab{b}}.

\bibitem[Hao et~al.(2021)Hao, Lattimore, Szepesv{\'a}ri, and
  Wang]{hao2021online}
Botao Hao, Tor Lattimore, Csaba Szepesv{\'a}ri, and Mengdi Wang.
\newblock Online sparse reinforcement learning.
\newblock In \emph{International Conference on Artificial Intelligence and
  Statistics}, pages 316--324. PMLR, 2021.

\bibitem[Javanmard and Montanari(2014)]{javanmard2014confidence}
Adel Javanmard and Andrea Montanari.
\newblock Confidence intervals and hypothesis testing for high-dimensional
  regression.
\newblock \emph{The Journal of Machine Learning Research}, 15\penalty0
  (1):\penalty0 2869--2909, 2014.

\bibitem[Kim and Paik(2019)]{kim2019doubly}
Gi-Soo Kim and Myunghee~Cho Paik.
\newblock Doubly-robust lasso bandit.
\newblock In \emph{Advances in Neural Information Processing Systems}, pages
  5869--5879, 2019.

\bibitem[Kirschner and Krause(2018)]{kirschner2018information}
Johannes Kirschner and Andreas Krause.
\newblock Information directed sampling and bandits with heteroscedastic noise.
\newblock In \emph{Conference On Learning Theory}, pages 358--384. PMLR, 2018.

\bibitem[Kirschner et~al.(2020{\natexlab{a}})Kirschner, Lattimore, and
  Krause]{kirschner2020information}
Johannes Kirschner, Tor Lattimore, and Andreas Krause.
\newblock Information directed sampling for linear partial monitoring.
\newblock In \emph{Conference on Learning Theory}, pages 2328--2369. PMLR,
  2020{\natexlab{a}}.

\bibitem[Kirschner et~al.(2020{\natexlab{b}})Kirschner, Lattimore, Vernade, and
  Szepesv{\'a}ri]{kirschner2020asymptotically}
Johannes Kirschner, Tor Lattimore, Claire Vernade, and Csaba Szepesv{\'a}ri.
\newblock Asymptotically optimal information-directed sampling.
\newblock \emph{arXiv preprint arXiv:2011.05944}, 2020{\natexlab{b}}.

\bibitem[Lattimore and Gy{\"o}rgy(2020)]{lattimore2020mirror}
Tor Lattimore and Andr{\'a}s Gy{\"o}rgy.
\newblock Mirror descent and the information ratio.
\newblock \emph{arXiv preprint arXiv:2009.12228}, 2020.

\bibitem[Lattimore and Szepesvari(2017)]{lattimore2017end}
Tor Lattimore and Csaba Szepesvari.
\newblock The end of optimism? an asymptotic analysis of finite-armed linear
  bandits.
\newblock In \emph{Artificial Intelligence and Statistics}, pages 728--737,
  2017.

\bibitem[Lattimore and Szepesv{\'a}ri(2019)]{lattimore2019information}
Tor Lattimore and Csaba Szepesv{\'a}ri.
\newblock An information-theoretic approach to minimax regret in partial
  monitoring.
\newblock In \emph{Conference on Learning Theory}, pages 2111--2139. PMLR,
  2019.

\bibitem[Lattimore and Szepesv{\'a}ri(2020)]{lattimore2018bandit}
Tor Lattimore and Csaba Szepesv{\'a}ri.
\newblock \emph{Bandit algorithms}.
\newblock Cambridge University Press, 2020.

\bibitem[Lattimore et~al.(2015)Lattimore, Crammer, and
  Szepesv{\'a}ri]{lattimore2015linear}
Tor Lattimore, Koby Crammer, and Csaba Szepesv{\'a}ri.
\newblock Linear multi-resource allocation with semi-bandit feedback.
\newblock In \emph{Advances in Neural Information Processing Systems}, pages
  964--972, 2015.

\bibitem[Liu et~al.(2018)Liu, Buccapatnam, and Shroff]{liu2018information}
Fang Liu, Swapna Buccapatnam, and Ness Shroff.
\newblock Information directed sampling for stochastic bandits with graph
  feedback.
\newblock In \emph{Proceedings of the AAAI Conference on Artificial
  Intelligence}, volume~32, 2018.

\bibitem[Lu and Van~Roy(2019)]{lu2019information}
Xiuyuan Lu and Benjamin Van~Roy.
\newblock Information-theoretic confidence bounds for reinforcement learning.
\newblock \emph{arXiv preprint arXiv:1911.09724}, 2019.

\bibitem[Mitchell and Beauchamp(1988)]{Mitchell1988}
T.~J. Mitchell and J.~J. Beauchamp.
\newblock Bayesian {V}ariable {S}election in {L}inear {R}egression.
\newblock \emph{Journal of the American Statistical Association}, 83\penalty0
  (404):\penalty0 1023--1032, 1988.

\bibitem[Oh et~al.(2020)Oh, Iyengar, and Zeevi]{oh2020sparsity}
Min-hwan Oh, Garud Iyengar, and Assaf Zeevi.
\newblock Sparsity-agnostic lasso bandit.
\newblock \emph{arXiv preprint arXiv:2007.08477}, 2020.

\bibitem[Ren and Zhou(2020)]{ren2020dynamic}
Zhimei Ren and Zhengyuan Zhou.
\newblock Dynamic batch learning in high-dimensional sparse linear contextual
  bandits.
\newblock \emph{arXiv preprint arXiv:2008.11918}, 2020.

\bibitem[Robbins and Monro(1951)]{RobbinsM1951}
Herbert Robbins and Sutton Monro.
\newblock A stochastic approximation method.
\newblock \emph{Annals of Mathematical Statistics}, 22:\penalty0 400--407,
  1951.

\bibitem[Ro{\v{c}}kov{\'a} and George(2018)]{rovckova2018spike}
Veronika Ro{\v{c}}kov{\'a} and Edward~I George.
\newblock The spike-and-slab lasso.
\newblock \emph{Journal of the American Statistical Association}, 113\penalty0
  (521):\penalty0 431--444, 2018.

\bibitem[Ro{\v r}kov{\'a} and George(2014)]{rovckova2014spike}
Veronika Ro{\v r}kov{\'a} and Edward~I. George.
\newblock {EMVS}: The {EM} {A}pproach to {B}ayesian variable selection.
\newblock \emph{Journal of the American Statistical Association}, 109\penalty0
  (506):\penalty0 828--846, 2014.

\bibitem[Rudelson and Zhou(2013)]{rudelson2013reconstruction}
Mark Rudelson and Shuheng Zhou.
\newblock Reconstruction from anisotropic random measurements.
\newblock \emph{IEEE Transactions on Information Theory}, 59\penalty0
  (6):\penalty0 3434--3447, 2013.

\bibitem[Rusmevichientong and Tsitsiklis(2010)]{rusmevichientong2010linearly}
Paat Rusmevichientong and John~N Tsitsiklis.
\newblock Linearly parameterized bandits.
\newblock \emph{Mathematics of Operations Research}, 35\penalty0 (2):\penalty0
  395--411, 2010.

\bibitem[Russo and Van~Roy(2014)]{russo2014learning}
Daniel Russo and Benjamin Van~Roy.
\newblock Learning to optimize via posterior sampling.
\newblock \emph{Mathematics of Operations Research}, 39\penalty0 (4):\penalty0
  1221--1243, 2014.

\bibitem[Russo and Van~Roy(2018)]{russo2018learning}
Daniel Russo and Benjamin Van~Roy.
\newblock Learning to optimize via information-directed sampling.
\newblock \emph{Operations Research}, 66\penalty0 (1):\penalty0 230--252, 2018.

\bibitem[Simchi-Levi and Xu(2020)]{simchi2020bypassing}
David Simchi-Levi and Yunzong Xu.
\newblock Bypassing the monster: A faster and simpler optimal algorithm for
  contextual bandits under realizability.
\newblock \emph{Available at SSRN}, 2020.

\bibitem[Tibshirani(1996)]{tibshirani1996}
R.~Tibshirani.
\newblock Regression shrinkage and selection via the lasso.
\newblock \emph{Journal of the Royal Statistical Society, Series B},
  58:\penalty0 267--288, 1996.

\bibitem[Vershynin(2009)]{vershynin2009role}
Roman Vershynin.
\newblock On the role of sparsity in compressed sensing and random matrix
  theory.
\newblock In \emph{2009 3rd IEEE International Workshop on Computational
  Advances in Multi-Sensor Adaptive Processing (CAMSAP)}, pages 189--192. IEEE,
  2009.

\bibitem[Wang et~al.(2018)Wang, Wei, and Yao]{wang2018minimax}
Xue Wang, Mingcheng Wei, and Tao Yao.
\newblock Minimax concave penalized multi-armed bandit model with
  high-dimensional covariates.
\newblock In \emph{International Conference on Machine Learning}, pages
  5200--5208, 2018.

\bibitem[Wang et~al.(2020)Wang, Chen, Fang, Wang, and Li]{wang2020nearly}
Yining Wang, Yi~Chen, Ethan~X Fang, Zhaoran Wang, and Runze Li.
\newblock Nearly dimension-independent sparse linear bandit over small action
  spaces via best subset selection.
\newblock \emph{arXiv preprint arXiv:2009.02003}, 2020.

\bibitem[Yang and Barron(1999)]{yang1999information}
Yuhong Yang and Andrew Barron.
\newblock Information-theoretic determination of minimax rates of convergence.
\newblock \emph{Annals of Statistics}, pages 1564--1599, 1999.

\end{thebibliography}
}

\newpage
\appendix

\section{Proofs}

\subsection{Proof of Claim \ref{claim:sub-opt}}\label{sec:proof_claim}

We first define the notion of restricted minimum eigenvalue.
\begin{definition}[Restricted minimum eigenvalue]\label{def:re}
Given a symmetric matrix $H\in\mathbb R^{d\times d}$ and integer $s\geq 1$, and $L>0$, the restricted minimum eigenvalue of $H$ is defined as
\begin{equation*}
    \phi^2(H, s, L):=\min_{\cS\subset [d], |\cS|\leq s}\min_{\theta\in\mathbb R^d}\Big\{\frac{\langle \theta, H\theta \rangle}{\|\theta_{\cS}\|_2^2}: \theta\in\mathbb R^d, \|\theta_{\cS^c}\|_1\leq L\|\theta_{\cS}\|_1\Big\}.
\end{equation*}
\end{definition}
Suppose $\{x^{(t)}\}_{t=1}^k\subseteq \mathbb R^d$ are $k$ independent random vectors who first $d-1$ coordinates are drawn uniformly from $\{-1, 1\}$ and the last coordinate is 1. Denote $\hat{\Sigma} = \sum_{t=1}^kx^{(t)}x^{(t)\top}$. It is easy to see $\mathbb E[\hat{\Sigma}]=I_{d}$ and $\sigma_{\min}(\mathbb E[\hat{\Sigma}])=1$. From the definition of restricted minimum eigenvalue, we have for any $L>0$,
\begin{equation*}
    \phi^2(\mathbb E[\hat{\Sigma}], s, L) \geq \sigma_{\min}^2(\mathbb E[\hat{\Sigma}]) = 1.
\end{equation*}
According to Theorem 10 in \cite{javanmard2014confidence} (essentially from Theorem 6 in \cite{rudelson2013reconstruction}), if the population covariance matrix 
satisfies the restricted eigenvalue condition, the empirical covariance matrix satisfies it as well with high
probability. Specifically, when $k= C_1 s \log(ed/s)$ for some large constant $C_1>0$,  the following holds:
\begin{equation*}
    \mathbb P\left(\phi^2(\hat{\Sigma}, s, 3)\geq \frac{1}{4}\right)\geq 1-2\exp(-k/C_1)\geq 0.5.
\end{equation*}
According to probabilistic argument, there exists a set of fixed actions $\{x^{(1)}, \ldots, x^{(k)}\}$ with  $k= C_1 s\log(ed/s)$ such that if we pull uniformly at random from them, the restricted minimum eigenvalue of the resulting covariance matrix is at least 1/4.

Next we compute how many rounds at most the optimism-based algorithm will choose from  informative action set $\cI$. Let $N_{t-1}(a)$ as the number of pulls for action $a$ until round $t$. Since we have $\theta^*\in\cC_t$ with high probability, then
\begin{equation*}
    \max_{a\in\cU}\max_{\tilde{\theta}\in\cC_t}\langle a, \tilde{\theta}\rangle\geq  \max_{a\in\cU} \langle a, \theta^*\rangle \geq s\varepsilon.
\end{equation*}
On the other hand for any action $a\in\cI$,
\begin{equation*}
\begin{split}
     \max_{\tilde{\theta}\in\cC_t}\langle a, \tilde{\theta}\rangle &= \max_{\tilde{\theta}\in\cC_t}\langle a, \tilde{\theta}-\theta^*\rangle+ \langle a, \theta^*\rangle \leq \max_{\tilde{\theta}\in\cC_t}\langle a, \tilde{\theta}-\theta^*\rangle+ \max_{a\in\cI}\langle a, \theta^*\rangle\\
     &=\max_{\tilde{\theta}\in\cC_t}\langle a, \tilde{\theta}-\theta^*\rangle+s\varepsilon-1\leq 2c\sqrt{\|a\|_{V_t^{-1}}s\log(n)}+s\varepsilon -1\\
     &\leq 2c\sqrt{\frac{s\log(n)}{N_{t-1}(a)}}+s\varepsilon-1.
\end{split}
\end{equation*}
If $N_{t-1}(a) >4c^2s\log(n)$ for $a\in\cI$, then we have $\max_{\tilde{\theta}\in\cC_t}\langle a, \tilde{\theta}\rangle < s\varepsilon$. Based on the optimism principle, the algorithm will switch to pull uninformative actions. This leads to the fact that optimism-based algorithm will pull at most $|\cI|4c^2s\log(n)$ rounds of information actions. According to the proof of minimax lower bound in \cite{hao2020high}, we have when $\sum_{a\in\cI}N_n(a)<1/(s\varepsilon^2)$, there exists another sparse parameter $\theta'$ such that \begin{equation}\label{eqn:case1}
      R_{\theta}(n) +  R_{\theta'}(n) \gtrsim ns\varepsilon \exp\left(-\frac{2n\varepsilon^2s^2}{d} \right) \,.
\end{equation}
By choosing $\varepsilon = \sqrt{1/(s^2\log(n)4c^2|\cI|)}$, we have for $d\geq n/(s\log(n)\log(ed/s))$
\begin{equation}\label{eqn:case1}
      R_{\theta}(n) +  R_{\theta'}(n) \gtrsim \frac{n}{\sqrt{\log(n)}|\cI|}\,.
\end{equation}
Note that $|\cI|=O(s\log(d/s))$ as we proved before. Then we can argue there exists a sparse linear bandit instance such that optimism-based algorithm will suffer linear regret for a data-poor regime. This ends the proof.

\subsection{Proof of Lemma  \ref{lemma:generic}}\label{sec:proof_main}

We decompose the Bayesian regret in terms of the instantaneous regret:
\begin{equation}\label{eqn:regret_decom}
\begin{split}
     \BR(n;\pi^{\text{IDS}}) &= \mathbb E\left[\sum_{t=1}^n\langle x^*, \theta\rangle - \sum_{t=1}^n Y_t\right] =\mathbb E\left[\sum_{t=1}^n\mathbb E_t\left[\langle x^*, \theta^*\rangle - Y_t\right]\right]\\
     &= \mathbb E\left[\sum_{t=1}^n\sum_a\mathbb E_t\left[\langle x^*, \theta^*\rangle - \langle a,\theta^*\rangle\right]\pi_t(a)\right]=\mathbb E\left[\sum_{t=1}^n\langle \pi_t, \Delta_t\rangle\right],
     \end{split}
\end{equation}
where the third equation is due to the zero mean of the noise.

We then bound one-step instantaneous regret. From the definition of $\pi_t$, we have
\begin{equation}\label{eqn:def_ir}
    \pi_t = \argmin_{\pi\in \cD(\cA)}\frac{\langle \pi, \Delta_t\rangle^2}{\langle \pi, I_t\rangle}.
\end{equation}
In addition, we denote 
\begin{equation}\label{eqn:def_gir}
    q_{\lambda, t} =\argmin_{\pi\in\cD(\cA)}\Psi_{t,\lambda}(\pi)= \argmin_{\pi\in \cD(\cA)}\frac{\langle \pi, \Delta_t \rangle^{\lambda}}{\langle \pi, I_t \rangle}.
\end{equation}
Note that 
\begin{equation*}
    \nabla_{\pi}\Psi_{t,2}(\pi) = \frac{2\langle \pi, \Delta_t \rangle\Delta_t}{\langle \pi, I_t\rangle} + \frac{\langle \pi, \Delta_t\rangle^2I_t}{\langle \pi, I_t\rangle^2}.
\end{equation*}
By the first-order optimality condition in Lemma \ref{lemma:first-order}, 
\begin{equation*}
    0\leq \left\langle \nabla_{\pi}\Psi_{t,2}(\pi_t), q_{\lambda, t}-\pi_t\right\rangle = \frac{2\langle q_{\lambda, t}-\pi_t, \Delta_t\rangle\langle \pi_t, \Delta_t\rangle}{\langle \pi_t, I_t\rangle} - \frac{\langle q_{\lambda, t}-\pi_t, I_t\rangle\langle \pi_t, \Delta_t\rangle^2 }{\langle \pi_t, I_t\rangle^2}.
\end{equation*}
This further implies
\begin{equation*}
   2\langle q_{\lambda, t}, \Delta_t\rangle\geq \langle \pi_t, \Delta_t\rangle\Big(1+\frac{\langle q_{\lambda,t}, I_t\rangle}{\langle \pi_t, I_t\rangle}\Big)\geq \langle \pi_t, \Delta_t\rangle.
\end{equation*}
Based on the above equation, we can bound the generalized information ratio as follows:
\begin{equation*}
    \begin{split}
        \frac{\langle \pi_t, \Delta_t\rangle^{\lambda}}{\langle \pi_t, I_t\rangle}&=\frac{\langle \pi_t,\Delta_t\rangle^2\langle \pi_t, \Delta_t\rangle^{\lambda-2}}{\langle \pi_t, I_t \rangle}\leq \frac{2^{\lambda-2}\langle \pi_t,\Delta_t\rangle^2\langle q_{\lambda, t}, \Delta_t\rangle^{\lambda-2}}{\langle \pi_t, I_t\rangle}\\
        &\leq \frac{2^{\lambda-2}\langle q_{\lambda,t}, \Delta_t\rangle^{\lambda-2}\langle q_{\lambda, t}, \Delta_t \rangle^2}{\langle q_{\lambda, t}, I_t\rangle} = 2^{\lambda-2}\min_{\pi\in\cD(\cA)}\frac{\langle \pi,\Delta_t\rangle^{\lambda}}{\langle \pi, I_t \rangle},
    \end{split}
\end{equation*}
where the first inequality is from Eq.~\eqref{eqn:def_ir} and the second inequality is from Eq.~\eqref{eqn:def_gir}.
According to the definition of $\Psi_{*,\lambda}$, we have
\begin{equation*}
   \langle \pi_t, \Delta_t\rangle\leq 2^{1-2/\lambda}\langle \pi_t, I_t\rangle^{1/\lambda}\Psi_{*,\lambda}^{1/\lambda}.
\end{equation*}

Next we prove $\langle \pi_t, I_t\rangle = I_t(x^*; (A_t, Y_t))$. By the chain rule of mutual information,
\begin{equation*}
    \begin{split}
        I_t(x^*;(A_t, Y_t)) &= I_t(x^*;A_t)+I_t(x^*; Y_t|A_t) = I_t(x^*; Y_t|A_t)\\
        &=\sum_{a\in\cA}\pi_t(a)I_t(x^*; Y_t|A_t=a),
    \end{split}
\end{equation*}
where we use the fact that $A_t$ and $x^*$ are independent. If $Z$ is independent of $X$ and $Y$, then we have $I(X;Y|Z) = I(X;Y)$. Since $A_t$ is independent of $x^*$ and $Y_t$ conditional on $\cF_t$, then 
\begin{equation*}
    \sum_{a\in\cA}\pi_t(a)I_t(x^*; Y_t|A_t=a) =  \sum_{a\in\cA}\pi_t(a)I_t(x^*; Y_{t,a}) = \langle \pi_t, I_t\rangle.
\end{equation*}
This proves the previous claim. Combining with Eq.~\eqref{eqn:regret_decom},
\begin{equation}\label{eqn:regret_decom_2}
\begin{split}
       \BR(n;\pi^{\text{IDS}}) &= \mathbb E\left[\sum_{t=1}^n\langle \pi_t, \Delta_t\rangle\right]\leq \mathbb E\left[\sum_{t=1}^n 2^{1-2/\lambda}I_t(x^*; (A_t, Y_t))^{1/\lambda}\Psi_{*,\lambda}^{1/\lambda}\right]\\
     &=2^{1-2/\lambda} \Psi_{*,\lambda}^{1/\lambda}\mathbb E\left[\sum_{t=1}^n I_t(x^*; (A_t, Y_t))^{1/\lambda}\right] \\
     &\leq 2^{1-2/\lambda} \Psi_{*,\lambda}^{1/\lambda}n^{1-1/\lambda}\mathbb E\left[\sum_{t=1}^n I_t(x^*; (A_t, Y_t))\right]^{1/\lambda},
\end{split}
\end{equation}
where the last inequality is from Holder's inequality with $p=\lambda/(\lambda-1)$ and $q=\lambda$.

In the end, we bound the cumulative information gain using the chain rule of mutual information,
\begin{equation*}
\begin{split}
     &\sum_{t=1}^n\mathbb E[I_t(x^*;(A_t, Y_t))]= \sum_{t=1}^n I(x^*;(A_t, Y_t)|\cF_t) =I(x^*;\cF_{n+1}).
\end{split}
\end{equation*}
Combining with Eq.~\eqref{eqn:regret_decom_2}, we have 
\begin{equation*}
        \BR_n(\pi, \rho)\leq 2^{1-2/\lambda} (\Psi_{*,\lambda}I(x^*; \cF_{n+1}))^{1/\lambda}n^{1-1/\lambda}. 
\end{equation*}
This ends the proof.

\subsection{Proof of Lemma \ref{lemma:bound_mutual_information}}\label{sec:proof_mutual_information}
Denote $Z_1 = (A_1, Y_1),\ldots, Z_n = (A_n, Y_n)$ such that $Z^n = (Z_1, \ldots, Z_n)$.
When the number of actions $K$ is small, we could directly bound it by
\begin{equation*}
    I(x^*; Z^n) = H(x^*) - H(x^*|Z^n)\leq H(x^*)\leq \log |\cA| = \log(K),
\end{equation*}
where for the first inequality we use the non-negativity of Shannon entropy. 

When the number of actions is large or infinite, we will bound it through the following information-theoretic argument. Recall that $x^* = \argmin_{a\in\cA} x^{\top}\theta^*$ so $x^*$ can be viewed as a deterministic function $\theta^*$. By the data processing lemma (Lemma \ref{lemma:data_processing}), we have $I(x^*;Z^n) \leq I(\theta^*; Z^n)$. In other words, we bound the information gain regarding the optimal action by the information gain regarding the true parameter.

Recall that we assume the prior distribution of $\theta^*$ is $\rho(\theta^*)$ that takes the value in $\Theta$.
From \cite{vershynin2009role}
, we know $\Theta$ enjoys an $\varepsilon$-net $\cN_{\varepsilon}$ under $\ell_2$-norm and its cardinality at most $(Cd/s\varepsilon)^s$ where $C$ is a constant. Hence, its metric entropy satisfies
\begin{equation}\label{eqn:metric_entropy}
    \log |\cN_{\varepsilon}|\leq s\log(Cd/s\varepsilon).
\end{equation}
 Suppose the Bayes mixture density $p_{\rho}(z^n) = \int_{\theta\in\Theta} p(z^n|\theta)d \rho(\theta)$. According to the definition of mutual information,
\begin{equation}\label{eqn:mutual_information}
\begin{split}
    I(\theta^*; Z^n) &=\mathbb E_{\theta^*}\left[D_{\KL}(\mathbb P_{Z^n|\theta^*}||\mathbb P_{Z^n})\right]\\
    &= \int_{\theta^*\in\Theta}\int p(z^n|\theta^*)\log\Big(\frac{p(z^n|\theta^*)}{p_w(z^n)}\Big)\mu(dz^n)d\rho(\theta^*)\\
    &\leq  \int_{\theta\in\Theta}\int p(z^n|\theta^*)\log\Big(\frac{p(z^n|\theta^*)}{q(z^n)}\Big)\mu(dz^n)d\rho(\theta^*)\\
    & = \int_{\theta\in\Theta} D_{\KL}(\mathbb P_{Z^n|\theta^*}||\mathbb Q_{Z^n})d\rho(\theta^*).
    \end{split}
\end{equation}
where the inequality is due to the fact that Bayes mixture density $p_{\rho}(z^n)$ minimizes the average KL divergences over any choice of densities $q(z^n)$. Then we choose ${\rho}_1$ as an uniform distribution over $\cN_{\varepsilon}$ such that $q(z^n) = p_{\rho_1}(z^n) = \int_{\theta\in\Theta} p(z^n|\theta)d \rho_1(\theta)$ and we denote $\mathbb Q_{Z^n}$ as the corresponding probability measure. Since $\cN_{\varepsilon}$ is an $\varepsilon$-net over $\Theta$ under $\ell_2$-norm, for each $\theta\in\Theta$, there exists $\tilde{\theta}\in\Theta$ such that $\|\theta-\tilde{\theta}\|_2\leq \varepsilon$.

To bound the KL-divergence term, we follow 
\begin{equation}\label{eqn:KL_calculation}
\begin{split}
    D_{\KL}(\mathbb P_{Z^n|\theta}||\mathbb Q_{Z^n}) &= \mathbb E\left[\log \frac{p(z^n|\theta^*)}{(1/|\cN_{\varepsilon}|)\sum_{\tilde{\theta}\in\cN_{\varepsilon}}p(z^n|\tilde{\theta})}\right]\\
    &\leq \mathbb E\left[\log \frac{p(z^n|\theta^*)}{(1/|\cN_{\varepsilon}|)p(z^n|\tilde{\theta})}\right]\\
    &\leq \log |\cN_{\varepsilon}| + D_{\KL}(\mathbb P_{Z^n|\theta}|| \mathbb P_{Z^n|\tilde{\theta}}).
\end{split}
\end{equation}
By the chain rule of KL-divergence,
\begin{equation*}
   D_{\KL}(P_{Z^n|\theta}||\mathbb P_{Z^n|\tilde{\theta}})\leq \mathbb E\left[ \sum_{t = 1}^n D_{\KL}(\mathbb P_{Y_t|A_t, Z^{t-1},\theta^*}||\mathbb P_{Y_t|A_t, Z^{t-1},\tilde{\theta}})\right],
\end{equation*}
where we define $Z^0 = \emptyset$. Under linear model Eq.~\eqref{def:sparse_linear} and bandit $\theta$, we know $Y_t\sim N(A_t^{\top}\theta, 1)$. A straightforward computation leads to 
\begin{equation}\label{eqn:KL_linear}
\begin{split}
     D_{\KL}(\mathbb P_{Y_t|A_t, Z^{t-1},\theta^*}||\mathbb P_{Y_t|A_t, Z^{t-1},\tilde{\theta}}) &= \frac{1}{2\sigma^2}\|A_t^{\top}\theta^* - A_t^{\top}\tilde{\theta}\|_2^2\\
     &\leq \frac{1}{2\sigma^2}\|A_t\|_{\infty}^2\|\theta^* - \tilde{\theta}\|_1^2\\
     &\leq \frac{1}{2\sigma^2}s\|\theta^* - \tilde{\theta}\|_2^2\\
     &\leq \frac{s}{2\sigma^2}\varepsilon^2,
\end{split}
\end{equation}
where the first inequality we use the fact that $\|a\|_{\infty}\leq 1$ and the parameters are sparse. Here actually we only require $\|a\|_{\infty}$ for $a\in\cA$ being bounded by a constant since evetually it will only appears inside the logarithm term. Putting Eqs. \eqref{eqn:metric_entropy}-\eqref{eqn:KL_linear} together, we have 
\begin{equation*}
    I(\theta^*;Z^n)\leq \int_{\theta^*\in \Theta} \Big(s\log(Cd/s\varepsilon)+ \frac{ns}{2\sigma^2}\varepsilon^2\Big) d\theta^* = s\log(Cd/s\varepsilon)+ \frac{ns}{2\sigma^2}\varepsilon^2.
\end{equation*}
With the choice of $\varepsilon = 1/\sqrt{n}$, we finally have 
\begin{equation*}
     I(\theta^*;Z^n)\leq 2s\log(Cdn^{1/2}/s).
\end{equation*}
This ends the proof.

\subsection{Proof of Lemma \ref{lemma:information_ratio}}
For any particular policy $\tilde{\pi}$, if one can derive an worse-case bound of $\Psi_{t,\lambda}(\tilde{\pi})$, we get an upper bound for $\Psi_{*,\lambda}$ automatically. The remaining step is to choose proper policy $\tilde{\pi}$.

First, we bound the information ratio with $\lambda=2$ that essentially follows Proposition 5 in \cite{russo2014learning} and Lemma 3 in \cite{russo2014learning} for a Gaussian noise. By the definition of mutual information, for any $a\in\cA$, we have 
\begin{equation}\label{eqn:IG_pinsker}
    \begin{split}
      I_t(x^*; Y_{t,a})
      &= D_{\KL}\left(\mathbb P_t((x^*, Y_{t,a})) ||\mathbb P_t(x^*\in\cdot)\mathbb P_t\big(Y_{t,a}\in\big)\right)\\
      &=\sum_{a^*\in\cA}\mathbb P_t(x^* = a^*)D_{\KL}\left(\mathbb P_t(Y_{t,a}=\cdot|x^*=a^*)||\mathbb P_t(Y_{t,a} = \cdot)\right).
    \end{split}
\end{equation}
Define $R_{\max}$ as the upper bound of maximum expected reward. It is easy to see $Y_{t,a}$ is a $\sqrt{R_{\max}^2+1}$ sub-Gaussian random variable. According to Lemma 3 in \cite{russo2014learning}, we have 
\begin{equation}\label{eqn:bound_information_gain}
    I_t(x^*;Y_{t,a})\geq \frac{2}{R_{\max}^2+1}\sum_{a^*\in\cA}\mathbb P_t(x^* = a^*)\Big(\mathbb E_t[Y_{t,a}|x^* = a^*]-\mathbb E_t[Y_{t,a}]\Big)^2.
\end{equation}
We bound the information ratio of IDS by the information ratio of TS:
\begin{equation*}
    \Psi_{*,2}\leq \max_{t\in[n]}\frac{\langle \pi_t^{\text{TS}}, \Delta_t\rangle^2}{\langle \pi_t^{\text{TS}}, I_t \rangle}.
\end{equation*}
Using the matrix trace rank trick described in Proposition 5 in \cite{russo2014learning}, we have $ \Psi_{*,2}\leq (R_{\max}^2+1)d/2$ in the end.

Second, we bound the information ratio with $\lambda=3$. Recall that the exploratory policy $\mu$ is defined as 
\begin{equation*}
      \max_{\mu\in\cD(\cA)}  \ \sigma_{\min}\Big(\int_{x\in\cA}xx^{\top} d \mu(x)\Big)\,.
\end{equation*}
Consider a mixture policy $\pi_t^{\mix} = (1-\gamma)\pi_t^{\text{TS}}+\gamma \mu$ where the mixture rate $\gamma\geq 0$ will be decided later. Then we will bound the following in two steps.
\begin{equation*}
    \Psi_{t,3}(\pi_t^{\mix}) = \frac{\langle \pi_t^{\mix}, \Delta_t \rangle^3}{\langle \pi_t^{\mix}, I_t\rangle}.
\end{equation*}

\paragraph{Step 1: Bound the information gain}
According the lower bound of information gain in Eq.~\eqref{eqn:bound_information_gain},
\begin{equation*}
\begin{split}
     \langle \pi_t^{\mix}, I_t\rangle&\geq \frac{2}{(R_{\max}^2+1)}\sum_{a\in\cA}\pi_t^{\mix}(a)\sum_{a^*\in\cA}\mathbb P_t(x^*=a^*)\left(\mathbb E_t[Y_{t,a}|x^*=a^*]-\mathbb E_t[Y_{t,a}]\right)^2\\
     & = \frac{2}{(R_{\max}^2+1)}\sum_{a\in\cA}\pi_t^{\mix}(a)\sum_{a^*\in\cA}\mathbb P_t(x^*=a^*)\left(a^{\top}\mathbb E_t[\theta^*|x^*=a^*]-a^{\top}\mathbb E_t[\theta^*]\right)^2.
\end{split}
\end{equation*}
By the definition of the mixture policy, we know that $\pi_t(a)\geq \gamma \mu(a)$ for any $a\in\cA$. Then we have 
\begin{equation*}
\begin{split}
     \langle \pi_t^{\mix}, I_t\rangle\geq& \frac{2}{(R_{\max}^2+1)}\gamma\sum_{a^*\in\cA}\mathbb P_t(x^*=a^*)\\
     &\cdot\sum_{a\in\cA}\mu(a)(\mathbb E_t[\theta^*|x^*=a^*]-\mathbb E_t[\theta^*])^{\top}aa^{\top}(\mathbb E_t[\theta^*|x^*=a^*]-\mathbb E_t[\theta^*]).
\end{split}
\end{equation*}
From the definition of minimum eigenvalue, we have
\begin{equation*}
     \langle \pi_t^{\mix}, I_t\rangle\geq \frac{2\gamma}{(R_{\max}^2+1)}\sum_{a\in\cA}\mathbb P_t(x^*=a)C_{\min}\left\|\mathbb E_t[\theta^*|x^*=a^*]-\mathbb E_t[\theta^*]\right\|_2^2.
\end{equation*}

\paragraph{Step 2: Bound the instant regret} We decompose the regret by the contribution from the exploratory policy and the one from TS:
\begin{equation}\label{eqn:bound1}
\begin{split}
      &\langle \pi_t^{\mix}, \Delta_t\rangle \\
      &= \sum_{a}\mathbb E_t\Big[\langle x^*, \theta^*\rangle-\langle a, \theta^* \rangle\Big]\pi_t^{\mix}(a),\\
      & = (1-\gamma)\sum_a \pi_t^{\text{TS}}(a)\mathbb E_t\Big[\langle x^*, \theta^*\rangle-\langle a, \theta^* \rangle\Big]+ \gamma \sum_a\mathbb E_t\Big[\langle x^*, \theta^* \rangle -\langle a,\theta^*\rangle\Big]\mu(a)\\
      &= (1-\gamma)\sum_a \mathbb P_t(x^*=a)\mathbb E_t\Big[\langle x^*, \theta^*\rangle-\langle a, \theta^* \rangle\Big]+ \gamma \sum_a\mathbb E_t\Big[\langle x^*, \theta^* \rangle -\langle a,\theta^*\rangle\Big]\mu(a)
\end{split}
\end{equation}
Since $R_{\max}$ is the upper bound of maximum expected reward, the second term can be bounded $2R_{\max}\gamma$. Next we bound the first term as follows:
\begin{equation*}
    \begin{split}
        &\sum_a \mathbb P_t(x^*=a)\mathbb E_t\Big[\langle x^*, \theta^*\rangle-\langle a, \theta^* \rangle\Big] \\
        &= \sum_a\mathbb P_t(x^*=a)\Big(\mathbb E_t[\langle a, \theta^*\rangle|x^*=a]-\mathbb E_t[\langle a,\theta^*\rangle]\Big)\\
        &=\sum_a\mathbb P_t^{1/2}(x^*=a)\mathbb P_t^{1/2}(x^*=a)\Big(\mathbb E_t[\langle a, \theta^*\rangle|x^*=a]-\mathbb E_t[\langle a,\theta^*\rangle]\Big)\\
        &\leq \sqrt{\sum_a \mathbb P_t(x^*=a)\Big(\mathbb E_t[\langle a, \theta^*\rangle|x^*=a]-\mathbb E_t[\langle a,\theta^*\rangle]\Big)^2},
    \end{split}
\end{equation*}
where we use Cathy-Schwarz inequality. Since all the optimal actions are sparse, any action $a$ with $\mathbb P_t(x^*=a)>0$ must be sparse. Then we have 
\begin{equation*}
    \left(a^{\top}(\mathbb E_t[\theta^*|x^*=a]-\mathbb E_t[\theta^*])\right)^2\leq s^2 \left\|\mathbb E_t[\theta^*|x^*=a^*]-\mathbb E_t[\theta^*]\right\|_2^2,
\end{equation*}
for any action $a$ with $\mathbb P_t(x^*=a)>0$. This further implies
\begin{equation}\label{eqn:bound2}
\begin{split}
     &\sum_a \mathbb P_t(x^*=a)\mathbb E_t\Big[\langle x^*, \theta^*\rangle-\langle a, \theta^* \rangle\Big]\\
     &\leq \sqrt{\sum_a \mathbb P_t(x^*=a)s^2\left\|\mathbb E_t[\theta^*|x^*=a^*]-\mathbb E_t[\theta^*]\right\|_2^2}\\
     &=\sqrt{\frac{s^2(R_{\max}^2+1)}{2\gamma C_{\min}}\frac{2\gamma}{(R_{\max}^2+1)}\sum_a \mathbb P_t(x^*=a)C_{\min}\left\|\mathbb E_t[\theta^*|x^*=a^*]-\mathbb E_t[\theta^*]\right\|_2^2}\\
     &\leq \sqrt{\frac{s^2(R_{\max}^2+1)}{2\gamma C_{\min}}\langle \pi_t^{\mix}, I_t \rangle}.
    \end{split}
\end{equation}
Putting Eq.~\eqref{eqn:bound1} and \eqref{eqn:bound2} together, we have
\begin{equation*}
    \langle \pi_t^{\mix}, \Delta_t \rangle\leq \sqrt{\frac{s^2(R_{\max}^2+1)}{2\gamma C_{\min}}\langle \pi_t^{\mix}, I_t\rangle} + 2R_{\max}\gamma.
\end{equation*}
By optimizing the mixture rate $\gamma$, we have 
\begin{equation*}
    \frac{\langle \pi_t^{\mix}, \Delta_t\rangle^3}{\langle \pi_t^{\mix}, I_t\rangle}\leq \frac{s^2(R_{\max}^2+1)}{8R^2_{\max}C_{\min}}\leq \frac{s^2}{4C_{\min}}.
\end{equation*}
This ends the proof.

\section{Detailed algorithms}\label{sec:detailed_algorithm}
For each $a\in\cA$, we expand $v_t(a)$ as follows:
\begin{equation*}
\begin{split}
     v_t(a) &= \text{Var}_t(\mathbb E_t[a^{\top}\theta|x^*]) =\mathbb E_t\Big[a^{\top}\mathbb E_t[\theta|x^*]-\mathbb E_t\big[a^{\top}\mathbb E_t[\theta|x^*]\big]\Big]^2\\
     &=\mathbb E_t\Big[a^{\top}\mathbb E_t[\theta|x^*]-a^{\top}\mathbb E_t[\theta]\Big]^2 = a^{\top}\mathbb E_t[(\mathbb E_t[\theta|x^*]-\mathbb E_t[\theta])(\mathbb E_t[\theta|x^*]-\mathbb E_t[\theta])^{\top}]a.
\end{split}
\end{equation*}
We denote $\mu_t = \mathbb E_t[\theta]$ as the posterior mean and $\mu_t^a = \mathbb E_t[\theta|x^*=a]$. We let $\Phi\in\mathbb R^{|\cA|\times d}$ as the feature matrix where each row of $\Phi$ represent each action in $\cA$. We summarize the procedure of estimating $\Delta_t, I_t$ in Algorithm \ref{alg:est_inf_ratio}.
{\small
\begin{algorithm}[htb!]
	\caption{Approximate $\Delta_t$, $v_t$ based on posterior samples}
	\begin{algorithmic}[1]\label{alg:est_inf_ratio}
		\STATE
		\textbf{Input:} $M$ posterior samples $\theta^1, \ldots, \theta^M$ from Eq.~\eqref{eqn:full_posterior}, action set $\cA$.
		\STATE 
		Calculate $\hat{\mu}_t = \sum_m \theta^m/M$.
\FOR{$a\in\cA$}		
		\STATE Find $\hat{\Theta}_a = \{m\in[M]:(\Phi\theta^m)_a = \max_{a'\in\cA}(\Phi\theta^m)_{a'}\}$.
		\STATE Calculate $\hat{p}^*_a = |\hat{\Theta}_a|/M$.
		\STATE Calculate $\hat{\mu}_t^a = \sum_{m\in\hat{\Theta}_a}\theta^m/|\hat{\Theta}_a|$. 
		\STATE Calculate
		\begin{equation*}
		    \hat{v}_t(a) = a^{\top}\sum_{a}\hat{p}^*_a(\hat{\mu}_t^a-\hat{\mu}_t)(\hat{\mu}_t^a-\hat{\mu}_t)^{\top}a, \hat{\Delta}_t(a) = \sum_{a\in\cA}\hat{p}^*_aa^{\top}\hat{\mu}^a_t - a^{\top}\hat{\mu}_t.
		    		\end{equation*}
		    \ENDFOR
		    \STATE \textbf{Output:} $\hat{v}_t, \hat{\Delta}_t$.
	\end{algorithmic}
\end{algorithm}
}

\section{Supporting lemmas}
\begin{lemma}[First-order optimality condition.]\label{lemma:first-order}
Suppose that $f_0$ in a convex optimization problem is differentiable. Let $\cX$ denote the feasible set. Then $x$ is optimal if and only if $x\in\cX$ and $\nabla f_0(x)^{\top}(y-x)\geq 0, \forall y\in\cX.$
\end{lemma}

\begin{lemma}[Data processing lemma]\label{lemma:data_processing}
If $Z=g(Y)$ for a deterministic function $g$, then $I(X;Y)\geq I(X;Z)$.
\end{lemma}

\end{document}